%% file: c3learning.tex
\documentclass[twoside,11pt]{article}

\usepackage{blindtext}

\usepackage{jmlr2e}
\usepackage{amsmath}
\usepackage{xcolor}
\usepackage{enumitem}

\definecolor{crimson}{rgb}{0.86, 0.08, 0.24}
\definecolor{crimsonglory}{rgb}{0.75, 0.0, 0.2}
\definecolor{denim}{rgb}{0.08, 0.38, 0.74}

\hypersetup{
    colorlinks=true,
    linkcolor=blue,
    filecolor=magenta,  
    urlcolor=magenta,
    citecolor=crimsonglory,
    pdftitle={Overleaf Example},
    pdfpagemode=FullScreen,
}

\input{macro}

\jmlrheading{1}{2000}{1-48}{4/00}{10/00}{}

\ShortHeadings{Capacity-constrained continual learning
}{}
\firstpageno{1}

\begin{document}

\title{Capacity-Constrained Continual Learning
}

\author{\name Zheng Wen \email zhengwen@alumni.stanford.edu \\
       \addr Google DeepMind\\
       1600 Amphitheatre Pkwy \\
       Mountain View, CA 94043, USA
       \AND
       \name Doina Precup \email doinap@google.com \\
       \addr Google DeepMind \\
       425, avenue Viger Ouest, bureau/suite 900 \\
       Montréal (Québec) H2Z 1W5
       \AND
       \name Benjamin Van Roy \email benvanroy@google.com \\
       \addr Google DeepMind\\
       1600 Amphitheatre Pkwy \\
       Mountain View, CA 94043, USA
       \AND
       \name Satinder Singh  \email baveja@google.com \\
       \addr Google DeepMind\\
       S2, 8 Handyside Street \\
       King's Cross London, N1C 4DJ
       }

\editor{}

\maketitle

\begin{abstract}
Any agents we can possibly build are subject to capacity constraints, as memory and compute resources are inherently finite. However, comparatively little attention has been dedicated to understanding how agents with limited capacity should allocate their resources for optimal performance. The goal of this paper is to shed some light on this question by studying a simple yet relevant continual learning problem: the capacity-constrained linear-quadratic-Gaussian (LQG) sequential prediction problem. We derive a solution to this problem under appropriate technical conditions. Moreover, for problems that can be decomposed into a set of sub-problems, we also demonstrate how to optimally allocate capacity across these sub-problems in the steady state. We view the results of this paper as a first step in the systematic theoretical study of learning under capacity constraints.

\end{abstract}


\section{Introduction}
\label{sec:introduction}

Recent progress in artificial intelligence has seen the rise of increasingly large models which, when trained on sufficient data, lead to the emergence of surprising capabilities. Yet any computational agents that we can possibly build are subject to computational constraints, as memory and compute resources are inherently finite. 
However, comparatively little attention has been dedicated in the machine learning and AI literature to understanding how agents with limited capacity should allocate their resources for optimal performance.

The goal of this paper is to shed some light on this question by focusing on agents that repeatedly observe their environment and make sequential predictions, but are limited in the amount of information that they can retain from their interaction. We are interested to understand how agents should decide which information to retain in this context. We focus on information, rather than hard memory constraints, because of the opportunity to formalize this type of constraint into a crisp mathematical problem which can be studied.

More precisely, we focus on the amount of information that an agent retains at each time $t$ from its observation history, $H_t$, into its internal (agent) state, $S_t$. This amount is measured by the \emph{mutual information} $\I(S_t; H_t)$ between the agent state $S_t$ and the observation history $H_t$. Intuitively, $H_t$ includes all the non-prior information available to the agent about its environment; and the agent state $S_t$ represents the agent's retained information. 
We consider the following capacity constraint at time $t$:
\begin{equation}
\label{eqn:mi-constraint}
\I(S_t; H_t) \leq B,
\end{equation}
where $B$ is an upper bound on the amount of retained information, measured in bits. Note that $\I(S_t; H_t)$ measures the \emph{average} amount of information the agent has retained from $H_t$. Thus, this constraint is a relaxation of the more realistic constraint that ``the agent's memory has $B$ bits and the representation of $S_t$ cannot exceed $B$ at any time". While the latter is a more natural way to phrase the constraint, it is difficult to study and manipulate mathematically. Furthermore, it is well known in information theory \citep{10.5555/1146355}, as discussed in \citet{kumar2023continual}, that at large scale, the hard memory constraint can be well-approximated by  (\ref{eqn:mi-constraint}).

Our goal in this paper is to develop precise intuitions and statements about the implications of constraint (\ref{eqn:mi-constraint}) on learning agents.
In order to takes steps towards this goal, we study a simple yet relevant continual learning problem: the Linear Quadratic Gaussian (LQG) sequential prediction problem with capacity constraint, an example for which analytical solutions can be derived.  LQG sequential prediction  is a classical problem that has been extensively studied in the fields of control theory and statistical learning. It is well known that without capacity constraints, an optimal solution to it can be derived based on Kalman filtering. In this paper, we derive an optimal solution to the LQG prediction problem {\em with capacity constraint} under appropriate technical conditions. We also analyze how the steady-state prediction loss scales with the capacity, as well as other characteristics of the underlying prediction problem, such as the mixing time and signal-to-noise ratio. Moreover, for prediction problems that can be decomposed into a set of sub-problems, we also demonstrate how to optimally allocate capacity across these sub-problems.
We hope that the insights gained from this didactic example can help and guide future agent design in more practical and challenging capacity-constrained continual learning problems.

The remainder of this paper is organized as follows. We first review relevant literature in Section~\ref{sec:literature}. Then, we formulate the capacity-constrained LQG continual learning problem, which we refer to as {\tt $\mathtt{C^3L}$-LQG} for short, in Section~\ref{sec:problem_formulation}. To illustrate how to solve {\tt $\mathtt{C^3L}$-LQG}, in Section~\ref{sec:stepping_stone}, we propose and solve the capacity-constrained LQG prediction problem (referred to as {\tt $\mathtt{C^2P}$-LQG}), which is a relaxation of {\tt $\mathtt{C^3L}$-LQG} without the incremental update constraints, and serves as a ``stepping stone" for the analysis. In Section~\ref{sec:opt-agent}, we derive solutions for  {\tt $\mathtt{C^3L}$-LQG}, under appropriate technical conditions. In Section~\ref{sec:steady-state}, we analyze {\tt $\mathtt{C^3L}$-LQG} in the steady state.
In Section~\ref{sec:optimal_capacity_allocation}, we consider problems that can be decomposed into a set of sub-problems, and study how to optimally allocate capacity across sub-problems; we illustrate this allocation through small computational examples. Finally, we conclude and discuss avenues for future work in Section~\ref{sec:conclusion}.

\section{Literature review}
\label{sec:literature}

Continual learning has been the focus of much recent work, even having a new conference dedicated to it. Much of this work is focused on non-stationary environments. But, continual learning is useful even in stationary settings. For example,  \citet{sutton2007role} demonstrates the importance of tracking the best solution in stationary environments, but where the agent is faced with partial observability. 

Our work is closely related to one paper in particular,  \citet{kumar2023continual}, which formalizes continual learning and casts it as reinforcement learning (RL) \citep{sutton2018reinforcement} with limited compute resources. Computational resources can be defined in various ways, but we focus on limitations of the expected capacity of a learner to retain information.

Bounded rationality has been also studied in relation to both natural
\citep{Simon1990, griffiths2020understanding}
and artificial agents \citep{russell1994provably}. In particular, \citet{russell1994provably} formulate the bounded agent design problem as a constrained optimization problem, which is close in spirit to what we will present.

One way to represent boundedness is through bottlenecks on the agent's capacity to retain information from its past. Learning is then carried out as a process that tries to retain past information that is relevant to predicting the future.
\citet{bialek2001predictability} studied the connections between \emph{predictive information} and the complexity of dynamics underlying a time series.
In follow-up work, \citet{chigirev2003optimal}
introduced an information-theoretic method for non-parametric, nonlinear dimensionality reduction, based on the infinite cluster limit of rate distortion theory. 

This paper focuses on the Linear Quadratic Gaussian sequential prediction problem with capacity constraint, which is closely related to the classic Kalman filtering problem \citep{kalman1960new, chui2017kalman}. Specifically, Kalman filtering can be viewed as a special case of our problem formulation, where the capacity is  infinite. Indeed, the results that we will present specialize to this case.

 The problem we consider can also be viewed as a special case of developing  capacity-constrained agents for HMMs and POMDPs. Prior research has studied capacity constraints for these types of problems, albeit of a different nature from the one we study here. For example, \citet{pajarinen2011periodic} and \citet{cubuktepe2021robust}  developed finite-state controllers for POMDPs, which can be viewed as a special case of capacity-constrained policy representations. Predictive state representations \citep{littman2001predictive, singh2012predictive} and Observable Operator Models \citep{jaeger2005efficient} have been proposed as a means to 
 learn a representation of the past interaction history with an HMM or POMDP which is finite. They use a specific finite-rank assumption on the dynamics model, which is a different notion of limited capacity.
 With this context, we will now turn to specifying the details of our problem formulation.
 
\section{Problem Formulation}
\label{sec:problem_formulation}

\subsection{Linear Gaussian sequential prediction}
\label{sec:sequential_prediction}

 Let $\left(\theta_t \right)_{t=0}^{\infty}$ denote a sequence of vectors in $\Re^d$. Assume that $\theta_0$ is drawn from $N(0, \Sigma_0)$, and that $\theta_t$ evolves according to the following auto-regressive model:
\begin{equation}
    \label{eq:theta_dynamics}
    \theta_{t+1} = A \theta_t + \omega_{t}, \quad \forall t=0,1,\ldots
\end{equation}
where $\omega_{t} \in \Re^d$ is drawn iid from a multivariate Gaussian distribution $N(0, \Sigma_\omega)$. Here, $\theta_t$ is the latent state of the system.

An agent aims to learn how to make accurate predictions of $\theta$ under this auto-regressive model, but it only receives noisy observations $Y_{t} \in \Re^m$, given by:
\begin{equation}
    \label{eq:outcome}
    Y_{t} = C \theta_t + \nu_{t},
\end{equation}
where $C \in \Re^{m \times d}$ and $\nu_{t}$ is drawn iid from the multivariate Gaussian observation noise distribution $N(0, \Sigma_\nu)$.

At each time $t=1, 2, \ldots$, the agent needs to make a prediction $\hat{\theta}_t$ of $\theta_t$, based on the sequence of observations $(Y_{t}), t=1, 2, \dots$'s, without ever observing the $\theta_t$'s.
As is traditional in control theory, we will assume that the agent knows the true model of the system, consisting of the dynamics matrix $A$, the observation matrix $C$, and the noise covariances $\Sigma_0$, $\Sigma_\omega$, and $\Sigma_\nu$

Let $H_t = \left(Y_1, Y_2, \ldots, Y_t \right)$ denote the ``history" of observations received by the agent by the time it needs to make the prediction $\hat{\theta}_{t} \in \Re^d$ of the latent vector $\theta_{t}$. We consider a setting where the agent chooses $\hat{\theta}_{t}$ based on $H_t$ and the known parameters $A$, $C$, $\Sigma_0$, $\Sigma_\omega$, and $\Sigma_\nu$. The agent's goal is to minimize the expected cumulative quadratic prediction error in the first $T$ time steps:
\begin{equation}
    \label{eq:objective}
    \min \sum_{t=1}^{T} \E \big [ \big \| \theta_t - \hat{\theta}_t \big \|_2^2 \big],
\end{equation}
which is also referred to as the total loss over the first $T$ steps.

\begin{remark}
A more general objective is to minimize a weighted quadratic loss,
\begin{equation}
    \label{eq:more-general-objective}
    \min \sum_{t=1}^{T} \E \big[ (\theta_t - \hat{\theta}_t)^\top P (\theta_t - \hat{\theta}_t) \big],
\end{equation}
where $P$ is a positive definite matrix. However, the objective in equation~(\ref{eq:more-general-objective}) can be reformulated as the objective in equation~(\ref{eq:objective}). To see this, note that if we define $\theta'_t = P^{1/2} \theta_t$ and $\hat{\theta}'_t = P^{1/2} \hat{\theta}_t$, then we have
\[
\big \| \theta'_t - \hat{\theta}'_t \big \|_2^2 = \big \| P^{1/2} \left( \theta_t - \hat{\theta}_t \right ) \big \|_2^2 = (\theta_t - \hat{\theta}_t)^\top P (\theta_t - \hat{\theta}_t).
\]
Therefore, all results obtained by studying equation~(\ref{eq:objective}) can be readily extended to the problem described by equation~(\ref{eq:more-general-objective})
\end{remark}

\subsection{Capacity-constrained continual learning}
\label{sec:capacity_constrained_agent}

The sequential prediction problem described in Section~\ref{sec:sequential_prediction} is a classical problem for which an analytical solution can be derived based on Kalman filtering (see Appendix~\ref{app:without_capacity_constraint}). In this paper, we study this problem but in the case of an  agent which performs   updates satisfying the following requirements:
\begin{enumerate}
    \item \textbf{Agent state:} the agent maintains and updates an \emph{agent state} $S_t$ at each time $t$. Moreover, the agent's prediction $\hat{\theta}_{t}$ depends on the history $H_t$ only through the agent state $S_t$, that is $\hat{\theta}_{t} \perp H_t \, | \, S_t$. Mathematically, this is equivalent to saying $\hat{\theta}_{t} \sim q_t(\cdot | S_t)$ for some conditional probability measure $q_t$. Without loss of generality, we assume that $S_0 = 0$.
    \item \textbf{Incremental updates:} the agent updates its agent state $S_{t+1}$ incrementally based on its previous agent state $S_t$, and new observation $Y_{t+1}$. That is, $ S_{t+1} \perp H_{t+1} \, | \, S_t, Y_{t+1} $.  Mathematically, this is equivalent to $S_{t+1} \sim p_t (\cdot | S_t, Y_{t+1})$ for some conditional probability measure $p_t$.
    \item \textbf{Capacity constraint:} the agent is subject to the capacity constraint described in the introduction:
    \[
    \I(S_t; H_t) \leq B, \forall t=0, 1, \dots T-1,
    \]
 where $\I(\cdot, \cdot)$ denotes the mutual information and $B \geq 0$ is the capacity limit.
\end{enumerate}

In summary, the capacity-constrained LQG continual learning problem, which is referred to as {\tt $\mathtt{C^3L}$-LQG}, is formulated as follows:
\begin{align}
    \label{opt:agent_design}
    \textstyle \min_{p_t, q_t} \quad & \, \textstyle \sum_{t=1}^T \E \big [ \big  \| \theta_t - \hat{\theta}_t \big \|_2^2 \big ] \\
    \text{s.t.} \quad &   \, S_{t+1} \sim p_t (\cdot | S_t, Y_{t+1}),  \; \, \forall t=0, \ldots, T-1 \nonumber \\
    & \, \hat{\theta}_{t} \sim q_t(\cdot | S_t),  \; \, \forall t=0, \ldots, T-1 \nonumber \\
    & \,\I(S_t; H_t) \leq B, \; \, \forall t=1, \ldots, T \nonumber \\
    & \, S_0 = 0 \nonumber
\end{align}

\section{Capacity-constrained prediction}
\label{sec:stepping_stone}

Before building agents for the {\tt $\mathtt{C^3L}$-LQG} problem formulated in Section~\ref{sec:capacity_constrained_agent}, we first consider an easier problem that serves as a ``stepping stone" for that problem. In particular, we consider a capacity-constrained prediction problem, referred to as {\tt $\mathtt{C^2P}$-LQG} and defined as
\begin{align}
    \label{opt:stepping_stone}
    \textstyle \min_{q_t} \quad & \, \textstyle \sum_{t=1}^T \E \big [ \big  \| \theta_t - \hat{\theta}_t \big \|_2^2 \big ] \\
    \text{s.t.} \quad &   \, \hat{\theta}_{t} \sim q_t(\cdot | H_t),   \; 
      \I(\hat{\theta}_t; H_t) \leq B, \; \, \forall t=1, \ldots, T
    \nonumber 
\end{align}
Note that {\tt $\mathtt{C^2P}$-LQG} has relaxed the incremental update constraint in {\tt $\mathtt{C^3L}$-LQG}.
Due to this relaxation, this problem decomposes over time $t$.
We can solve this problem analytically and the solution is
\[
\hat{\theta}_t = F_t \bar{\theta}_t + \epsilon_t,
\]
where $\bar{\theta}_t = \E \left[ \theta_t \, \middle | \, H_t\right] \in \Re^d$, $F_t \in \Re^{d \times d}$ is a positive-definite matrix depending on $\Cov[\bar{\theta}_t ]$ and capacity $B$, and $\epsilon_t \sim N(0, \Psi_t)$, where $\Psi_t$ is a covariance matrix also depending on $\Cov[\bar{\theta}_t ]$ and capacity $B$. Note that we need $\epsilon_t$ to be independent of noises in the environment; however, $\epsilon_t$'s can depend on each other across time.

Specifically, the solution to {\tt $\mathtt{C^2P}$-LQG} is as follows: let $\lambda_{t1}, \ldots, \lambda_{td}$ denote the $d$ eigenvalues of $\Cov[\bar{\theta}_t]$, and let $U_t \in \Re^{d \times d}$ denote the orthogonal matrix encoding the eigenvectors of $\Cov[\bar{\theta}_t]$. Define $\eta_t$ as the unique solution of equation
\begin{equation}
\sum_{i=1}^d \left[\log (2 \lambda_{ti} / \eta) \right]^+ = 2B,
\end{equation}
and define
\[
B_{ti} = \frac{1}{2} \left[ \log (2 \lambda_{ti}/\eta_t) \right]^+ \quad  \forall i=1,2,\ldots, d.
\]
To simplify the exposition, we also define two diagonal matrices $D^{\mathrm{F}}_t, D^{\mathrm{\Psi}}_t \in \Re^{d \times d}$ as
\begin{align}
D^{\mathrm{F}}_t = & \, \mathrm{diag} \left( 1 - \exp(-2B_{t1}),  \ldots, 1 - \exp(-2B_{td}) \right) \nonumber \\
D^{\mathrm{\Psi}}_t = & \, \mathrm{diag} \left(
\left[1 - \exp(-2B_{t1}) \right] \exp(-2B_{t1}) \lambda_{t1}, 
\cdots,
\left[1 - \exp(-2B_{td}) \right] \exp(-2B_{td}) \lambda_{td}
\right).
\label{eqn:D_F_D_Psi}
\end{align}
Then the solution to {\tt $\mathtt{C^2P}$-LQG}
is given in the following theorem:
\begin{theorem}[Optimal capacity-constrained prediction]
\label{thm:optimal_cc_prediction}
One optimal solution to the capacity-constrained prediction problem ({\tt $\mathtt{C^2P}$-LQG}) is
\begin{equation}
\label{eqn:optimal-prediction}
\hat{\theta}_t = F_t \bar{\theta}_t + \epsilon_t,
\end{equation}
where $\bar{\theta}_t = \E \left[ \theta_t \, \middle | \, H_t\right] \in \Re^d$, $F_t = U_t D_t^{\mathrm{F}} U_t^\top$, and $\epsilon_t \sim N(0, \Psi_t)$ with
$\Psi_t = U_t D^{\mathrm{\Psi}}_t U_t^\top$. Note that $U_t$ is the orthogonal matrix encoding the eigenvectors of $\Cov[\bar{\theta}_t ]$, and $D^{\mathrm{F}}_t$ and $D^{\mathrm{\Psi}}_t$ are defined in equation~\ref{eqn:D_F_D_Psi}.
\end{theorem}
\begin{proof}
Notice that the capacity-constrained prediction problem decomposes over time $t$. Thus, we only need to consider the one-step optimization problem
\begin{align}
    \textstyle \min_{q_t} \quad & \,  \E \big [ \big  \| \theta_t - \hat{\theta}_t \big \|_2^2 \big ] \\
    \text{s.t.} \quad &   \, \hat{\theta}_{t} \sim q_t(\cdot | H_t),   \; 
      \I(\hat{\theta}_t; H_t) \leq B, 
    \nonumber 
\end{align}
for all $t=1,2, \ldots, T$. Furthermore, note that with $\bar{\theta}_t = \E \left[ \theta_t | H_t \right ]$, we have
\[
\E \big [ \big  \| \theta_t - \hat{\theta}_t \big \|_2^2 \big ] =
\E \big [ \big  \| \theta_t - \bar{\theta}_t \big \|_2^2 \big ] +
\E \big [ \big  \| \bar{\theta}_t - \hat{\theta}_t \big \|_2^2 \big ].
\]
Since $\E \big [ \big  \| \theta_t - \bar{\theta}_t \big \|_2^2 \big ]$ does not depend on $\hat{\theta}_t$, it is equivalent to solve the following optimization problem:
\begin{align}
    \label{opt:stepping_stone_2}
    \textstyle \min_{q_t} \quad & \,  \E \big [ \big  \| \bar{\theta}_t - \hat{\theta}_t \big \|_2^2 \big ] \\
    \text{s.t.} \quad &   \, \hat{\theta}_{t} \sim q_t(\cdot | H_t),   \; 
      \I(\hat{\theta}_t; H_t) \leq B .
    \nonumber 
\end{align}
Since $\bar{\theta}_t = \E [\theta_t | H_t]$ is a deterministic function of $H_t$, thus, from the data processing inequality, we have
\[
\I (\hat{\theta}_t ; \bar{\theta}_t) \leq \I(\hat{\theta}_t ; H_t).
\]
Hence, a relaxation of the optimization problem (\ref{opt:stepping_stone_2}) is
\begin{align}
    \label{opt:stepping_stone_3}
    \textstyle \min_{q_t} \quad & \,  \E \big [ \big  \| \bar{\theta}_t - \hat{\theta}_t \big \|_2^2 \big ] \\
    \text{s.t.} \quad &   \, \hat{\theta}_{t} \sim q_t(\cdot | H_t),   \; 
      \I(\hat{\theta}_t; \bar{\theta}_t) \leq B .
    \nonumber 
\end{align}
Since both the objective and the constraint $\I(\hat{\theta}_t; \bar{\theta}_t) \leq B$ only depend on the joint distribution of $\big( \hat{\theta}_t, \bar{\theta}_t \big)$, we can also replace the constraint $ \hat{\theta}_{t} \sim q_t(\cdot | H_t)$ with  $ \hat{\theta}_{t} \sim q_t(\cdot | \bar{\theta}_t)$. Consequently, the optimization problem (\ref{opt:stepping_stone_3}) is exactly the optimization problem defining the Gaussian distortion-rate function, and the optimal objective value is $D \left(B, \Cov[\bar{\theta}_t] \right)$. Please refer to Appendix~\ref{app:gaussian-dr} for details.

Finally, we prove that one optimal solution for optimization problem (\ref{opt:stepping_stone_3}) is also feasible for optimization problem (\ref{opt:stepping_stone_2}), consequently, it is also  a solution for optimization problem (\ref{opt:stepping_stone_2}). To see it, note that from Theorem~\ref{thm:gaussian-dr} in Appendix~\ref{app:gaussian-dr}, one optimal solution for optimization problem (\ref{opt:stepping_stone_3}) is $\hat{\theta}_t = F_t \bar{\theta}_t + \epsilon_t$, as is defined in equation~\ref{eqn:optimal-prediction}. Note that with this solution, we have $\hat{\theta}_t \perp H_t | \bar{\theta}_t$, thus $\I (\hat{\theta}_t ; \bar{\theta}_t) = \I(\hat{\theta}_t ; H_t)$ and this solution is also feasible for optimization problem (\ref{opt:stepping_stone_2}).
\end{proof}

As Theorem~\ref{thm:optimal_cc_prediction} illustrates, an optimal solution to {\tt $\mathtt{C^2P}$-LQG} (\ref{opt:stepping_stone}) is a \emph{linear Gaussian} agent, in the sense that $\hat{\theta}_t$ is a linear function of $\bar{\theta}_t$, with additive multivariate Gaussian noise.

\section{Optimal linear Gaussian agent}
\label{sec:opt-agent}

we now revisit the capacity-constrained continual learning problem ({\tt $\mathtt{C^3L}$-LQG}) described in Section~\ref{sec:capacity_constrained_agent}, based on the 
results and insights developed in Section~\ref{sec:stepping_stone}. We are particularly interested when an optimal solution to the capacity-constrained continual learning problem (\ref{opt:agent_design}) is a \emph{linear Gaussian} agent. In particular, we say an continual learning agent is a linear Gaussian agent if (1) $S_{t+1}$ is a linear function of $(S_t, Y_{t+1})$, with additive multivariate Gaussian noise; and (2) $\hat{\theta}_t$ is also a linear function of $S_t$, with additive multivariate Gaussian noise.

The following theorem provides a sufficient condition under which an optimal solution to the {\tt $\mathtt{C^3L}$-LQG} problem (\ref{opt:agent_design}) is a linear Gaussian agent.

\begin{theorem}[Optimal linear Gaussian agent]
\label{thm:optimal-linear-agent}
For all $t=1,2, \ldots, T$, define
\[
F_t = U_t D^{\mathrm{F}}_t U_t^\top, \quad \Psi_t = U_t D^{\mathrm{\Psi}}_t U_t^\top, \quad \textrm{and} \quad 
P_t = A \Cov[\theta_t | H_t] A^\top + \Sigma_\omega,
\]
where $U_t$ is the orthogonal matrix encoding the eigenvectors of $\Cov[\bar{\theta}_t ]$, and $D^{\mathrm{F}}_t$ and $D^{\mathrm{\Psi}}_t$ are defined in equation~\ref{eqn:D_F_D_Psi}.
If (1) $F_t$ is invertible for all $t$, and (2) for all $t=1,2, \ldots, T-1$,
\begin{multline}
\label{eq:incremental-update-cond}
     \left(I - P_t C^\top (C P_t C^\top + \Sigma_\nu)^{-1} C  \right) A F_t^{-1} \Psi_t F_t^{-1} A^\top 
    \left(I -  C^\top (C P_t C^\top + \Sigma_\nu)^{-1} C P_t  \right)  \\
    \leq F_{t+1}^{-1} \Psi_{t+1} F_{t+1}^{-1},
\end{multline}
then an optimal solution to the {\tt $\mathtt{C^3L}$-LQG} problem (\ref{opt:agent_design}) is a linear Gaussian agent, and the optimal total cost is the same as the relaxed capacity-constrained prediction problem (\ref{opt:stepping_stone}).
\end{theorem}
\begin{proof}
We prove this theorem by showing that the solution to the capacity-constrained prediction problem (\ref{opt:stepping_stone}) derived in Theorem~\ref{thm:optimal_cc_prediction} admits an incremental update under the conditions of this theorem.
Notice that from Kalman filtering (see Appendix~\ref{app:without_capacity_constraint}), we have
\[
\bar{\theta}_{t+1} = \left(I - P_t C^\top (C P_t C^\top + \Sigma_\nu)^{-1} C  \right) A \bar{\theta}_t +  P_t C^\top (C P_t C^\top + \Sigma_\nu)^{-1} Y_{t+1}.
\]
From Theorem~\ref{thm:optimal_cc_prediction}, an optimal solution to problem (\ref{opt:stepping_stone}) is $\hat{\theta}_t = F_t \bar{\theta}_t + \epsilon_t$. Since we assume that $F_t$ is invertible, we have $\bar{\theta}_t = F_t^{-1} \big ( \hat{\theta}_t - \epsilon_t \big)$. Consequently, we have
\begin{multline}
F_{t+1} \bar{\theta}_{t+1} + F_{t+1} \left(I - P_t C^\top (C P_t C^\top + \Sigma_\nu)^{-1} C  \right) A F_t^{-1}  \epsilon_t 
= \\
F_{t+1} \left(I - P_t C^\top (C P_t C^\top + \Sigma_\nu)^{-1} C  \right) A F_t^{-1} \hat{\theta}_t  + F_{t+1} P_t C^\top (C P_t C^\top + \Sigma_\nu)^{-1} Y_{t+1}.
\end{multline}
Since $\hat{\theta}_{t+1} = F_{t+1} \bar{\theta}_{t+1} + \epsilon_{t+1}$, thus, we can compute $\hat{\theta}_{t+1}$ based on $\hat{\theta}_t$ and $Y_{t+1}$ as long as 
\begin{equation}
\label{eq:incremental-update-cond-1}
\Cov \left[  F_{t+1} \left(I - P_t C^\top (C P_t C^\top + \Sigma_\nu)^{-1} C  \right) A F_t^{-1}  \epsilon_t \right] \leq \Cov \left[\epsilon_{t+1} \right] = \Psi_{t+1}.
\end{equation}
To see it, notice that if the above condition holds, we can independently sample 
\[
\epsilon_{t+1}' \sim N \left( 0,  \Psi_{t+1} - \Cov \left[  F_{t+1} \left(I - P_t C^\top (C P_t C^\top + \Sigma_\nu)^{-1} C  \right) A F_t^{-1}  \epsilon_t \right] \right),
\]
and compute
\begin{multline}
\hat{\theta}_{t+1} = F_{t+1} \bar{\theta}_{t+1} + F_{t+1} \left(I - P_t C^\top (C P_t C^\top + \Sigma_\nu)^{-1} C  \right) A F_t^{-1}  \epsilon_t + \epsilon_{t+1}'
= \\
F_{t+1} \left(I - P_t C^\top (C P_t C^\top + \Sigma_\nu)^{-1} C  \right) A F_t^{-1} \hat{\theta}_t  + F_{t+1} P_t C^\top (C P_t C^\top + \Sigma_\nu)^{-1} Y_{t+1} + \epsilon_{t+1}'. \nonumber
\end{multline}
Notice that condition (\ref{eq:incremental-update-cond-1}) is exactly the same as the condition (\ref{eq:incremental-update-cond}). This concludes the proof.
\end{proof}

It is worth mentioning that the sufficient conditions in Theorem~\ref{thm:optimal-linear-agent} are easy to verify. It is easy to compute $F_t$, $\Psi_t$, and $P_t$ by definition, and then numerically verify the sufficient conditions. In Section~\ref{sec:scalar_case}, we will prove that the sufficient conditions always hold in the \emph{scalar case} with $d=m=1$. On the other hand, whether or not the conditions in Theorem~\ref{thm:optimal-linear-agent} are also necessary is still an open problem, and we leave it to future work.

\subsection{Scalar case}
\label{sec:scalar_case}

We now provide an example in which the sufficient conditions in Theorem~\ref{thm:optimal-linear-agent} always hold, and consequently an optimal capacity-constrained learning agent is a linear Gaussian agent. Specifically, we show that for the scalar case with $d=m=1$, as long as $B>0$, the sufficient conditions in  Theorem~\ref{thm:optimal-linear-agent} always hold.

To see it, note that in the scalar case, all the matrices have reduced to scalars. Also, the covariance matrices have reduced to the variances. Note that from Section~\ref{sec:stepping_stone}, for the scalar case we have $U_t=1$ and
\begin{equation}
\label{eq:F_Psi_scalar}
F_t = 1 - \exp(-2B) \quad \text{and} \quad \Psi_t = \left[ 1 - \exp(-2B)\right] \exp(-2B) \Var \left( \bar{\theta}_t \right),
\end{equation}
for all $t$. Consequently, as long as $B>0$, $F_t \neq 0$ (i.e. $F_t$ is invertible). We now prove that condition (\ref{eq:incremental-update-cond}) holds in the scalar case.
Notice that in the scalar case, this condition reduces to
\[
\left [ 
F_{t+1} \left ( \frac{\Sigma_\nu}{P_t C^2 + \Sigma_\nu} \right) A F_t^{-1}
\right ]^2 \Psi_t \leq \Psi_{t+1}.
\]
From equation~\ref{eq:F_Psi_scalar}, the above inequality reduces to
\begin{equation}
\left [ 
\left ( \frac{\Sigma_\nu}{P_t C^2 + \Sigma_\nu} \right) A 
\right ]^2 \Var \left( \bar{\theta}_t \right) \leq \Var \left( \bar{\theta}_{t+1} \right).
\label{eq:scalar_cond}
\end{equation}
Moreover, from the law of the total variance, we have
\[
\Var \left[\theta_t \right] = \Var \left[\bar{\theta}_t \right] + \E [M_t] = \Var \left[\bar{\theta}_t \right] + M_t ,
\]
since $M_t = \Cov[\theta_t | H_t]$ is deterministic based on Kalman filtering. Moreover,
we have
\[
\Var \left[ \theta_t \right ] = A^{2t} \Sigma_0 +  \frac{1 - A^{2t}}{1-A^2} \Sigma_\omega.
\]
Also, from Kalman filtering, we have
\[
M_{t+1} = \frac{\Sigma_\nu}{P_t C^2 + \Sigma_\nu} \left( A^2 M_t + \Sigma_\omega \right).
\]
We now prove inequality~(\ref{eq:scalar_cond}), notice that
\begin{align}
   \left[ \left ( \frac{\Sigma_\nu}{P_t C^2 + \Sigma_\nu} \right) A 
\right ]^2 \Var \left( \bar{\theta}_t \right)  \stackrel{(a)}{\leq} & \, 
\left ( \frac{\Sigma_\nu}{P_t C^2 + \Sigma_\nu} \right) A^2 \Var \left( \bar{\theta}_t  \right)  \nonumber \\
\stackrel{(b)}{=}& \, \left ( \frac{\Sigma_\nu}{P_t C^2 + \Sigma_\nu} \right) A^2 \left[  A^{2t} \Sigma_0 +  \frac{1 - A^{2t}}{1-A^2} \Sigma_\omega  - M_t \right] \nonumber \\
=& \, \left ( \frac{\Sigma_\nu}{P_t C^2 + \Sigma_\nu} \right) 
 \left[  A^{2(t+1)} \Sigma_0 +  \frac{A^2 - A^{2(t+1)}}{1-A^2} \Sigma_\omega  - A^2M_t \right] \nonumber \\
 =& \, \left ( \frac{\Sigma_\nu}{P_t C^2 + \Sigma_\nu} \right) 
 \left[  A^{2(t+1)} \Sigma_0 +  \frac{1 - A^{2(t+1)}}{1-A^2} \Sigma_\omega  - (A^2M_t + \Sigma_\omega) \right] \nonumber \\
 \stackrel{(c)}{=}& \, \left ( \frac{\Sigma_\nu}{P_t C^2 + \Sigma_\nu} \right) 
 \left[ \Var \left[\theta_{t+1} \right]  - (A^2M_t + \Sigma_\omega) \right] \nonumber \\
 \stackrel{(d)}{=}& \, \left ( \frac{\Sigma_\nu}{P_t C^2 + \Sigma_\nu} \right)   \Var \left[\theta_{t+1} \right] - M_{t+1}  \nonumber \\
 \stackrel{(e)}{\leq}& \, \Var \left[\theta_{t+1} \right] - M_{t+1} = \Var \left[\bar{\theta}_{t+1} \right],
\end{align}
where (a) and (e) follow from $\frac{\Sigma_\nu}{P_t C^2 + \Sigma_\nu^2} \leq 1$, (b) and (c) follow from the definition of $\Var[\theta_t]$ and $\Var[\bar{\theta}_t]$, and (d) follows from Kalman filtering. Thus, the sufficient conditions in Theorem~\ref{thm:optimal-linear-agent} always hold in the scalar case.

\section{Capacity-constrained continual learning in steady state}
\label{sec:steady-state}

It is well known that if $A$ is asymptotically stable, i.e. all the eigenvalues of $A$ have a magnitude less than $1$, then as time $t \rightarrow \infty$, the joint distribution of $(\theta_t, Y_t)$ will converge to a steady-state distribution. Specifically, the steady-state distribution of $\theta_{\infty}$ is $N(0, \Sigma)$, where $\Sigma$ is the unique solution to the discrete-time Lyapunov equation:
\begin{equation}
    \label{eqn:lyapunov}
    \Sigma = A \Sigma A^\top + \Sigma_{\omega}.
\end{equation}
Similarly, the posterior covariance $M_t = \Cov[\theta_t | H_t]$ also converges to a steady-state value $M$. It can be computed as follows: first, we compute the steady-state forward covariance matrix $P = \lim_{t \rightarrow \infty} \Cov[\theta_{t+1} | H_t]$, 
which is a solution to the discrete-time algebraic Riccati equation:
\begin{equation}
    \label{eqn:riccati}
    A P A^\top - P + \Sigma_{\omega} -
    A P C^\top (C P C^\top + \Sigma_\nu)^{-1} C P A^\top = 0.
\end{equation}
Then, $M$ can be computed as
\begin{equation}
    M = P - PC^\top (CPC^\top + \Sigma_\nu)^{-1} CP.
\end{equation}
Consequently, from the law of total covariance, the steady-state value of $\Cov[\bar{\theta}_t]$ is $\Sigma-M$.

\subsection{Steady-state capacity-constrained prediction}
\label{sec:ss-ccp}

We now revisit the capacity-constrained prediction problem ({\tt $\mathtt{C^2P}$-LQG}) studied in Section~\ref{sec:stepping_stone} in steady state. Similarly, let $\lambda_{1}, \ldots, \lambda_{d}$ denote the $d$ eigenvalues of $\lim_{t \rightarrow \infty} \Cov[\bar{\theta}_t] = \Sigma - M$, and let $U \in \Re^{d \times d}$ denote the orthogonal matrix encoding the eigenvectors of $\Sigma-M$. Define $\eta$ as the unique solution of equation
$
\sum_{i=1}^d \left[\log (2 \lambda_{i} / \eta) \right]^+ = 2B
$,
and define
$
B_{i} = \frac{1}{2} \left[ \log (2 \lambda_{i}/\eta) \right]^+$, for all $i=1,2,\ldots, d$.
We define two matrices, $F$ and $\Psi$, as
\begin{align}
\label{eqn:F_Psi_ss}
F = & \, U \mathrm{diag} \left( 1 - \exp(-2B_{1}),  \ldots, 1 - \exp(-2B_{d}) \right) U^\top  \\
\Psi = & \, U \mathrm{diag} \left(
\left[1 - \exp(-2B_{1}) \right] \exp(-2B_{1}) \lambda_{1}, 
\cdots,
\left[1 - \exp(-2B_{d}) \right] \exp(-2B_{d}) \lambda_{d}
\right) U^\top. \nonumber
\end{align}

Then, from Theorem~\ref{thm:optimal_cc_prediction}, we have the following result:
\begin{theorem}[Steady-state capacity-constrained prediction]
\label{thm:ss-ccp}
One optimal solution to the steady-state capacity-constrained prediction problem ({\tt $\mathtt{C^2P}$-LQG}) is
\[
\hat{\theta}_t = F \bar{\theta}_t + \epsilon_t,
\]
where $\bar{\theta}_t = \E [\theta_t | H_t]$ and $\epsilon_t \sim N(0, \Psi)$, and $F$ and $\Psi$ are defined in equation~\ref{eqn:F_Psi_ss}. Moreover, the optimal asymptotic cost is
\begin{equation}
\label{eqn:optimal_asym_cost}
\trace(M) + \sum_{i=1}^d \exp(-2B_i) \lambda_i.
\end{equation}
\end{theorem}
\begin{proof}
Note that the first result directly follows from Theorem~\ref{thm:optimal_cc_prediction}. We now prove the second result. Note that
\[
\E \left[  \| \theta_t - \hat{\theta}_t  \|_2^2 \right] =
\E \left[  \| \theta_t - \bar{\theta}_t  \|_2^2 \right]
+ 
\E \left[  \| \bar{\theta}_t - \hat{\theta}_t  \|_2^2 \right].
\]
Moreover, note that in the steady state
\begin{align}
    \E \left[  \| \theta_t - \bar{\theta}_t  \|_2^2 \right] =& \, \E \left[ \E \left[  \| \theta_t - \bar{\theta}_t  \|_2^2 \, \middle | \, H_t \right] \right] \nonumber \\
    =& \,  \trace \left( \E \left[ \E \left[ (\theta_t - \bar{\theta}_t) (\theta_t - \bar{\theta}_t)^\top  \, \middle | \, H_t \right] \right] \right ) = \, \trace \left( M \right ), \nonumber
\end{align}
where the last equality follows from $ \E \left[ (\theta_t - \bar{\theta}_t) (\theta_t - \bar{\theta}_t)^\top  \, \middle | \, H_t \right] = M$ in the steady state.

 On the other hand, since in the steady state $\hat{\theta}_t = F \bar{\theta}_t + \epsilon_t$, we have
 $\bar{\theta}_t-\hat{\theta}_t = (I-F) \bar{\theta}_t - \epsilon_t$. Thus,
 \begin{align}
     \E \left[  \| \bar{\theta}_t - \hat{\theta}_t  \|_2^2 \right] =& \,  \E \left[  \| (I-F) \bar{\theta}_t - \epsilon_t  \|_2^2 \right] \nonumber \\
     =& \, \trace \left ( 
     \E \left[  [ (I-F) \bar{\theta}_t - \epsilon_t ]
     [ (I-F) \bar{\theta}_t - \epsilon_t ]^\top \right]
     \right ) \nonumber \\
     =& \, \trace \left ( 
     [I-F] \Cov(\bar{\theta}_t) [I-F] + \Psi
     \right ). \nonumber
 \end{align}
Note that $F$ is symmetric and hence $F^\top=F$. For simplicity of exposition, define
\begin{align}
    \Lambda = & \, \mathrm{diag} \left(\lambda_1, \ldots, \lambda_d \right )
    \nonumber \\
    D_F = & \, \mathrm{diag} \left( 1 - \exp(-2B_{1}),  \ldots, 1 - \exp(-2B_{d}) \right) \nonumber \\
    D_\Psi = & \, \mathrm{diag} \left(
\left[1 - \exp(-2B_{1}) \right] \exp(-2B_{1}) \lambda_{1}, 
\cdots,
\left[1 - \exp(-2B_{d}) \right] \exp(-2B_{d}) \lambda_{d}
\right). \nonumber
\end{align}
Then we have
\begin{align}
[I-F] \Cov(\bar{\theta}_t) [I-F] + \Psi =& \, [U U^\top- U D_F U^\top] U \Lambda U^\top [U U^\top-U D_F U^\top] + U D_\Psi U^\top \nonumber \\
=& U \left[ (I-D_F) \Lambda (I-D_F) + D_\Psi \right] U^\top.
\nonumber
\end{align}
From the cyclic property of the matrix trace, we have
\begin{align}
\trace \left([I-F] \Cov(\bar{\theta}_t) [I-F] + \Psi \right) =& \, \trace \left( U \left[ (I-D_F) \Lambda (I-D_F) + D_\Psi \right] U^\top \right)  \nonumber \\
=& \, \trace \left(  \left[ (I-D_F) \Lambda (I-D_F) + D_\Psi \right] U^\top U \right)  \nonumber \\
=& \, \trace \left(  \left[ (I-D_F) \Lambda (I-D_F) + D_\Psi \right] \right)  \nonumber \\
=& \, \sum_{i=1}^d \exp(-2B_i) \lambda_i.
\end{align}
This concludes the proof.
\end{proof}

\begin{remark}
Consider two special cases: $B=\infty$ and $B=0$. We now show that the optimal asymptotic cost derived in Theorem~\ref{thm:ss-ccp} recovers the optimal asymptotic cost in these two special cases. To see it, note that when $B=\infty$, $\hat{\theta}_t = \bar{\theta}_t$ and hence the optimal asymptotic cost is $\trace(M)$. On the other hand, when $B=\infty$, $B_i=\infty$ for all $i$. Thus, 
equation~\ref{eqn:optimal_asym_cost} in Theorem~\ref{thm:ss-ccp} also reduces to $\trace(M)$.

On the other hand, when $B=0$, $\hat{\theta}_t$ is independent of $H_t$ and the optimal asymptotic cost is $\trace(\Sigma)$. On the other hand, when $B=0$, $B_i=0$ for all $i$, and equation~\ref{eqn:optimal_asym_cost} reduces to
$\trace(M) + \sum_{i=1}^d \lambda_i = \trace(M) + \trace(\Sigma-M) = \trace(\Sigma)$. 
\end{remark}

\subsection{Steady-state capacity-constrained continual learning}

We now consider the capacity-constrained continual learning in the steady state. In particular, we identify a sufficient condition under which the optimal solution to the steady-state capacity-constrained prediction problem (Theorem~\ref{thm:ss-ccp}) admits a linear incremental update. Consequently, an an optimal solution to the {\tt $\mathtt{C^3L}$-LQG} problem (\ref{opt:agent_design}) in the steady state is a linear Gaussian agent, and the asymptotic optimal total cost is the same as the steady-state capacity-constrained prediction problem considered in Section~\ref{sec:ss-ccp}.

Recall that for the steady-state capacity-constrained prediction problem, 
$\hat{\theta}_t = F \bar{\theta}_t + \epsilon_t$, where $\epsilon_t \sim N(0, \Psi)$ and $F$ and $\Psi$ are defined in equation~\ref{eqn:F_Psi_ss}. Without loss of generality, we assume that for some $m=0,1,\ldots, d$, $B_1, B_2, \ldots, B_m >0$, but $B_{m+1}, B_{m+2}, \ldots, B_d = 0$. We also define $U_1 \in \Re^{d \times m}$ as the first $m$ columns of $U$, and $U_2 \in \Re^{d \times (m-d)}$ as the last $m-d$ columns of $U$. In other words, $U_1$ encodes the eigenvectors of $\lim_{t \rightarrow \infty} \Cov [\bar{\theta}_t] = \Sigma-M$ associated with eigenvalues $\lambda_1, \ldots, \lambda_m$, and $U_2$ encodes the eigenvectors of $\Sigma-M$ associated with eigenvalues $\lambda_{m+1}, \ldots, \lambda_d$. To simplify the exposition, we also define a diagonal matrix $D \in \Re^{m \times m}$ as
\begin{equation}
    \label{eqn:d-matrix}
    D = \, \mathrm{diag} \left( \frac{\exp(-2B_1) \lambda_1}{1-\exp(-2B_1)}, \ldots, \frac{\exp(-2B_m) \lambda_1}{1-\exp(-2B_m)} \right) \in \Re^{m \times m}
\end{equation}
As a shorthand notation, we also define
\[
K=\left(I - P C^\top (C P C^\top + \Sigma_\nu)^{-1} C  \right) A \quad \text{and} \quad L= P C^\top (C P C^\top + \Sigma_\nu)^{-1}.
\]
Consequently, in the steady state, the Kalman filter reduces to
\begin{equation}
\bar{\theta}_{t+1} = K \bar{\theta}_t + L Y_{t+1}. \label{eqn:kalman_update_short}
\end{equation}
We then have the following theorem:
\begin{theorem}
\label{thm:ss-ccc}
With $U_1$, $D$ and $K$ defined above,
if 
\begin{enumerate}[label=(\Alph*)]
\item $\mathrm{range}(K^\top U_1) \subseteq \mathrm{range}(U_1)$ and 
\item $U_1^\top K U_1 D U_1^\top K^\top U_1 \leq D$
\end{enumerate}
then an optimal solution to the {\tt $\mathtt{C^3L}$-LQG} problem (\ref{opt:agent_design}) in the steady state is a linear Gaussian agent, and the asymptotic optimal total cost is the same as the steady-state capacity-constrained prediction problem.
\end{theorem}
\begin{proof}
To simplify the exposition, define diagonal matrices $D_{F, 1}, D_{\Psi, 1} \in \Re^{m \times m}$ as
\begin{align}
    D_{F,1} = & \, \mathrm{diag} \left( 1 - \exp(-2B_{1}),  \ldots, 1 - \exp(-2B_{m}) \right) \nonumber \\
    D_{\Psi, 1} = & \, \mathrm{diag} \left(
\left[1 - \exp(-2B_{1}) \right] \exp(-2B_{1}) \lambda_{1}, 
\cdots,
\left[1 - \exp(-2B_{m}) \right] \exp(-2B_{m}) \lambda_{m}
\right) . \nonumber
\end{align}
It is straightforward to verify that
$\hat{\theta}_t = U_1 D_{F, 1} U_1^\top \bar{\theta}_t + \epsilon_t$ and $\hat{\theta}_t = U_1 U_1^\top \hat{\theta}_t$. Consequently, we have
\[
U_1^\top \hat{\theta}_t = D_{F, 1} U_1^\top \bar{\theta}_t + U_1^\top \epsilon_t, \quad \text{and} \quad \Cov \left [U_1^\top \epsilon_t \right] = D_{\Psi, 1}.
\]
From the Kalman filter, we also have
\[
U_1^\top \bar{\theta}_{t+1} = U_1^\top K \bar{\theta}_t +  U_1^\top L Y_{t+1}
= U_1^\top K \left[U_1 U_1^\top \bar{\theta}_t + U_2 U_2^\top \bar{\theta}_t \right]+  U_1^\top L Y_{t+1}.
\]
Recall that $\mathrm{range}(U_1) \perp \mathrm{range}(U_2)$, since $\mathrm{range}(K^\top U_1) \subseteq \mathrm{range}(U_1)$ by condition (A) in the theorem, we have $\mathrm{range}(K^\top U_1) \perp \mathrm{range}(U_2)$. So we have $U_1^\top K U_2 = 0$. Thus
\[
U_1^\top \bar{\theta}_{t+1} 
= U_1^\top K U_1 U_1^\top \bar{\theta}_t +  U_1^\top L Y_{t+1}.
\]
Since $U_1^\top \bar{\theta}_t = D_{F,1}^{-1} \left[ U_1^\top \hat{\theta}_t - U_1^\top \epsilon_t \right]$, we have
\[
D_{F, 1} U_1^\top \bar{\theta}_{t+1} 
= D_{F, 1} U_1^\top K U_1 D_{F,1}^{-1} \left[ U_1^\top \hat{\theta}_t - U_1^\top \epsilon_t \right]  +  D_{F, 1} U_1^\top L Y_{t+1}, 
\]
that is
\[
D_{F, 1} U_1^\top K U_1 D_{F,1}^{-1}  U_1^\top \hat{\theta}_t   +  D_{F, 1} U_1^\top L Y_{t+1} = D_{F, 1} U_1^\top  \bar{\theta}_{t+1} + D_{F, 1} U_1^\top K U_1 D_{F,1}^{-1}  U_1^\top \epsilon_t .  
\]
Since $U_1^\top \hat{\theta}_{t+1} = D_{F, 1} U_1^\top \bar{\theta}_{t+1} + U_1^\top \epsilon_{t+1}$, thus, as long as
\begin{equation}
\label{eqn:ss-incremental-update-condition}
\Cov \left[D_{F, 1} U_1^\top K U_1 D_{F,1}^{-1} U_1^\top \epsilon_t \right] \leq \Cov \left[ U_1^\top \epsilon_{t+1} \right] = D_{\Psi, 1},
\end{equation}
we can incrementally compute $U_1^\top \hat{\theta}_{t+1}$ as
\[
U_1^\top \hat{\theta}_{t+1} = D_{F, 1} U_1^\top K U_1 D_{F,1}^{-1}  U_1^\top \hat{\theta}_t   +  D_{F, 1} U_1^\top L Y_{t+1} + \tilde{\epsilon}_{t+1},
\]
with $\tilde{\epsilon}_{t+1}$ independently sampled from $N \left(0, D_{\Psi, 1} - \Cov \left[D_{F, 1} U_1^\top K U_1 D_{F,1}^{-1} U_1^\top \epsilon_t \right] \right)$.
Notice that condition~\ref{eqn:ss-incremental-update-condition} is equivalent to
\[
 U_1^\top K U_1 D_{F, 1}^{-1} D_{\Psi, 1} D_{F, 1}^{-1} U_1^\top K^\top U_1 \leq D_{F, 1}^{-1} D_{\Psi, 1}  D_{F, 1}^{-1}.
\]
Note that $D = D_{F, 1}^{-1} D_{\Psi, 1}  D_{F, 1}^{-1}$, this is equivalent to condition (B) in the theorem. This concludes the proof.
\end{proof}

\begin{remark}
We now briefly motivate and explain the two conditions in Theorem~\ref{thm:ss-ccc}. Note that from Kalman filtering, $\bar{\theta}_t$ admits linear incremental updates (see equation~\ref{eqn:kalman_update_short}). Moreover, from Theorem~\ref{thm:ss-ccp}, $\hat{\theta}_t = F \bar{\theta}_t + \epsilon_t$ is a linear transformation of $\bar{\theta}_t$, with additive Gaussian noises. To ensure that $\hat{\theta}_t$ also admits an incremental update, we need to ensure that (1) $F \bar{\theta}_t$ admits an incremental update and (2) the noise ``magnitude" of $\epsilon_t$, measured in covariance matrix, decays after the incremental update. Note that (1) does not always hold since $F$ is not full-rank in general. Roughly speaking, condition (A) in Theorem~\ref{thm:ss-ccc} ensures (1) while condition (B) ensures (2).  
\end{remark}

\begin{remark}
It is straightforward to rewrite the conditions of Theorem~\ref{thm:ss-ccc} as follows: there exists a matrix $G \in \Re^{m \times m}$ s.t. $K^\top U_1 = U_1 G$ and $G^\top D G \leq D$.
\end{remark}

\subsection{Diagonal systems}

Before concluding this section, we would like to show that the conditions in Theorem~\ref{thm:ss-ccc} always hold in diagonal systems. Specifically, in a diagonal system, 
$A = \diag (a_1, a_2, \ldots, a_d)$, $C = \diag (c_1, c_2, \ldots, c_d)$, $\Sigma_\omega = \diag (\sigma^2_{\omega,1}, \sigma^2_{\omega,2}, \ldots, \sigma^2_{\omega,d})$, and 
$\Sigma_\nu = \diag (\sigma^2_{\nu,1}, \sigma^2_{\nu,2}, \ldots, \sigma^2_{\nu,d})$ are all diagonal matrices. To ensure that $A$ is asymptotically stable, we require $a_i^2 < 1$ for all $i$.

It is straightforward to show that $P$, $K$ and $L$ are also diagonal matrices. Moreover, $K = \diag (k_1, k_2, \ldots, k_d)$ with
\[
k_i = \left(1 - \frac{p_i c_i^2}{p_i c_i^2 + \sigma_{\nu, i}^2} \right) a_i = \frac{\sigma_{\nu, i}^2}{p_i c_i^2 + \sigma_{\nu, i}^2} a_i,
\]
where $p_i$ is the $i$-th diagonal element in the diagonal matrix $P$. Note that $|k_i| \leq |a_i| < 1$ for all $i$.

Let $e_i \in \Re^d$ denote the vector whose $i$-th element is $1$ and all other elements are zero. Without loss of generality, we assume that $U_1 = [e_1, \ldots, e_m]$. It is straightforward to verify that 
\[
K^\top U_1 = K U_1 = [k_1 e_1, k_2 e_2, \ldots k_m e_m] \in \Re^{d \times m} ,
\]
thus we have $\mathrm{range}(K^\top U_1) \subseteq \mathrm{range}(U_1)$. Moreover, the condition 
$U_1^\top K U_1 D U_1^\top K^\top U_1 \leq D$ reduces to
\[
|k_i| \leq 1 \qquad \forall i=1,2,\ldots, m,
\]
which has been proven above.

\section{Optimal capacity allocation}
\label{sec:optimal_capacity_allocation}

In this section, we consider problems that can be decomposed into a set of sub-problems, and study how to optimally allocate capacity across sub-problems. We illustrate this allocation through small computational examples.

\subsection{Optimal asymptotic cost in scalar systems}
\label{sec:ss-scalar}

We first compute the optimal asymptotic cost for the scalar case with $d=m=1$. Note that for the scalar case, we have $F=1-\exp(-2B)$ and $\Psi = [\exp(-2B)-\exp(-4B)] [\Sigma-M]$. On the other hand, we have
\[
\E \left[ (\theta_t - \hat{\theta}_t)^2 \right] = \E \left[ (\theta_t - \bar{\theta}_t)^2 \right] +
\E \left[ (\bar{\theta}_t - \hat{\theta}_t)^2 \right].
\]
Moreover, as $t \rightarrow \infty$, we have
\[
\E \left[ (\theta_t - \bar{\theta}_t)^2 \right] = \E \left[ \E \left [ (\theta_t - \bar{\theta}_t)^2 \middle | \, H_t \right] \right] = \E \left[ M \right] = M.
\]
On the other hand, as $t \rightarrow \infty$, we have
\begin{align}
\E \left[ (\bar{\theta}_t - \hat{\theta}_t)^2 \right] = & \, \E \left[ (\bar{\theta}_t - F \bar{\theta}_t - \epsilon_t)^2 \right] = \E \left[ ( \exp(-2B) \bar{\theta}_t - \epsilon_t)^2 \right] \nonumber \\
=& \, \exp(-4B) [\Sigma-M] + \Psi \nonumber \\
=& \, \exp(-4B) [\Sigma-M] + [\exp(-2B)-\exp(-4B)] [\Sigma-M] \nonumber \\
=& \, \exp(-2B)[\Sigma-M].
\end{align}
Consequently, the optimal total cost is $M + \exp(-2B)[\Sigma-M]$. 

We now explicitly compute $\Sigma$ and $M$ in the scalar case. Note that in the scalar case, the Lyapunov equation is $\Sigma = A^2 \Sigma + \Sigma_\omega$ and the Riccati equation is
\[
A^2 P - P + \Sigma_\omega - \frac{A^2 C^2 P^2}{C^2P +\Sigma_\nu} = 0.
\]
Consequently, $\Sigma = \frac{\Sigma_\omega}{1-A^2}$ and 
\begin{equation}
\label{eqn:scalar_ss_P}
P = \frac{C^2 \Sigma_\omega + (A^2-1) \Sigma_\nu + \sqrt{\left(C^2 \Sigma_\omega + (A^2-1) \Sigma_\nu \right)^2 + 4 C^2 \Sigma_\omega \Sigma_\nu}}{2 C^2}.
\end{equation}
Moreover, note that $P=A^2M+\Sigma_\omega$, thus we have
\begin{equation}
\label{eqn:scalar_ss_M}
M = \frac{ (A^2-1) \Sigma_\nu -C^2 \Sigma_\omega + \sqrt{\left [  (A^2-1) \Sigma_\nu - C^2 \Sigma_\omega \right ]^2 + 4 A^2 C^2 \Sigma_\omega \Sigma_\nu}}{2 A^2 C^2}.
\end{equation}
Thus, we have the following result:
\begin{corollary}
For the scalar case with $d=m=1$, the optimal asymptotic cost is
\[
M + \exp(-2B)[\Sigma-M],
\]
where $\Sigma = \frac{\Sigma_\omega}{1-A^2}$ and $M$ is defined in equation~\ref{eqn:scalar_ss_M}.
\end{corollary}

\subsection{Optimal capacity allocation in diagonal systems}
\label{sec:diagonal}

We now consider the optimal capacity allocation in diagonal systems. Recall that in a diagonal system, 
$A = \diag (a_1, a_2, \ldots, a_d)$, $C = \diag (c_1, c_2, \ldots, c_d)$, $\Sigma_\omega = \diag (\sigma^2_{\omega,1}, \sigma^2_{\omega,2}, \ldots, \sigma^2_{\omega,d})$, and 
$\Sigma_\nu = \diag (\sigma^2_{\nu,1}, \sigma^2_{\nu,2}, \ldots, \sigma^2_{\nu,d})$ are all diagonal matrices. Based on result in Section~\ref{sec:ss-scalar}, for each $i=1,2, \ldots, d$, we define
\begin{align}
    \Sigma_i =& \, \frac{\sigma^2_{\omega, i}}{1-a_i^2} \nonumber \\
    M_i =& \, \frac{ (a_i^2-1) \sigma^2_{\nu,i} -c_i^2 \sigma^2_{\omega,i} + \sqrt{\left [  (a_i^2-1) \sigma^2_{\nu, i} - c_i^2 \sigma^2_{\omega,i} \right ]^2 + 4 a_i^2 c_i^2 \sigma^2_{\omega,i} \sigma^2_{\nu,i}}}{2 a_i^2 c_i^2}
\end{align}
for $i=1,2,\ldots, d$. Thus, the optimal capacity allocation in a diagonal system with a given capacity $B$ can be obtained by solving the following problem:
\begin{align}
    \min_{\mathbf{B}} \quad & \, \sum_{i=1}^d \exp(-2 B_i) \left[ \Sigma_i - M_i \right] \\
    \mathrm{s.t.} \quad & \, \sum_{i=1}^d B_i = B \nonumber \\ & \, B_i \geq 0 \quad \forall i=1,2, \ldots, d \nonumber 
\end{align}
where $\mathbf{B}=(B_1, B_2, \ldots, B_d)$. Note that this is a strictly convex optimization problem.  Let $\mathbf{B}^*=(B_1^*, B_2^*, \ldots, B_d^*)$ denote the optimal solution to this problem, then the optimal asymptotic cost is
\[
\sum_{i=1}^d \left[M_i + \exp(-2B_i^*)[\Sigma_i-M_i] \right].
\]

\subsection{Experiment results}
\label{sec:experiment}

We now demonstrate some experiment results for optimal capacity allocation in diagonal systems.
We consider three different cases: in the first case, the subsystems in the diagonal system differ in the system mixing time; in the second case, the subsystems differ in the magnitudes of the system noises; in the last case, the subsystems differ in the observation signal-to-noise (SNR) ratios. The optimal capacity allocations are respectively illustrated in Figures~\ref{fig:case_1_alloc}, \ref{fig:case_2_alloc} and \ref{fig:case_3_alloc}, while the optimal asymptotic costs are illustrated in Figures~\ref{fig:case_1_cost}, \ref{fig:case_2_cost}, and \ref{fig:case_3_cost}.


\subsubsection*{Case 1: Different system mixing times}

We first consider a setting where the subsystems differ in their system mixing times. Specifically, we choose 
 $A=\mathrm{diag}([0.99 I_{100}, 0.95 I_{100}, 0.9 I_{100}])$, $C=I_{300}$, $\Sigma_\omega=I_{300}$, $\Sigma_\nu=I_{300}$, where $I_{100}$ is an identify matrix with dimension $d=300$ and $I_{100}$ is defined similarly. Note that all the scalar subsystems in this diagonal system have the same $c_i=1$, $\sigma^2_{\omega, i}=1$, and $\sigma^2_{\nu, i}=1$, but have different $a_i$'s. Note that larger $a_i$ (in magnitude) corresponds to longer system mixing time. Note that optimal capacity allocation is illustrated in Figure~\ref{fig:case_1_alloc} while the optimal asymptotic cost is illustrated in Figure~\ref{fig:case_1_cost}.

\begin{figure}[ht]
    \centering
    \includegraphics[scale=0.7]{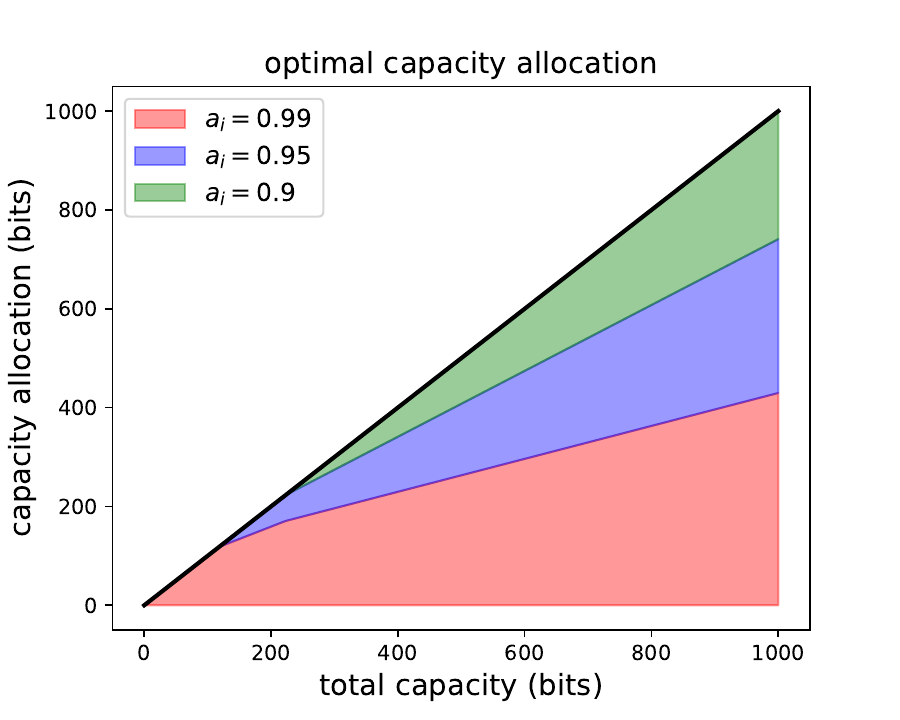}
    \caption{Optimal capacity allocation for Case 1}
    \label{fig:case_1_alloc}
\end{figure}

\begin{figure}[ht]
    \centering
    \includegraphics[scale=0.7]{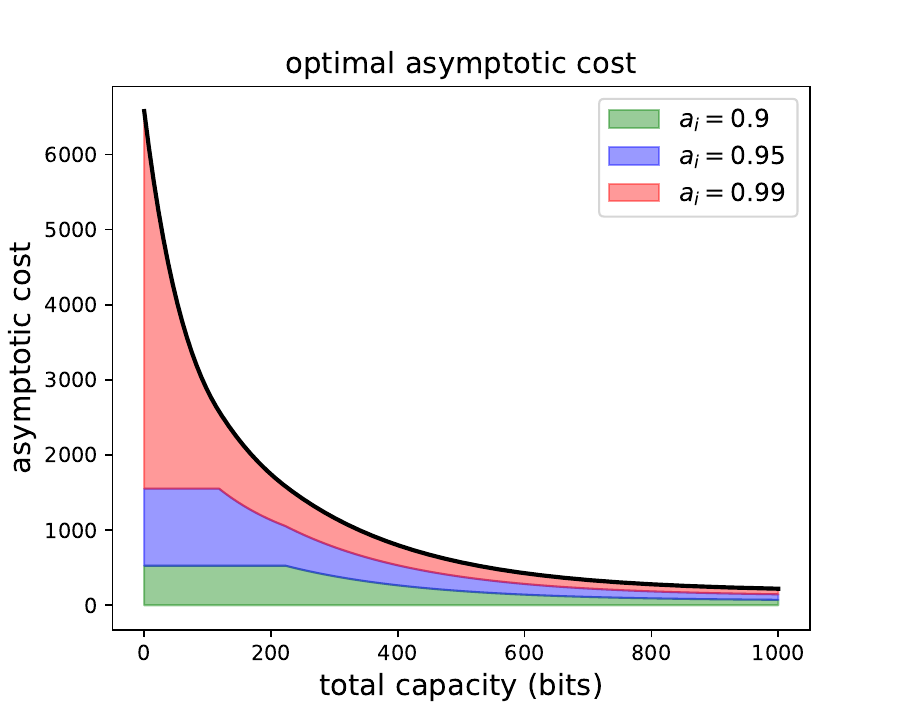}
    \caption{Optimal asymptotic loss for Case 1}
    \label{fig:case_1_cost}
\end{figure}

\subsubsection*{Case 2: Different system noise magnitudes}

We then consider a setting where the subsystems differ in the magnitudes of the system noises, measured in the variance. Specifically, we choose
$A=0.95 I_{300}$, $C=I_{300}$, $\Sigma_\omega=\mathrm{diag}([10 I_{100}, 3 I_{100}, I_{100}])$, and $\Sigma_\nu=I_{300}$. Note that all scalar subsystems have the same $a_i=0.95$, $c_i=1$, $\sigma_{\nu, i}^2=1$, but have different $\sigma^2_{\omega,i}$'s. Note that larger $\sigma^2_{\omega,i}$ incurs larger asymptotic cost.
The optimal capacity allocation in this case is illustrated in Figure~\ref{fig:case_2_alloc} while the optimal asymptotic cost is illustrated in Figure~\ref{fig:case_2_cost}.

\begin{figure}[ht]
    \centering
    \includegraphics[scale=0.7]{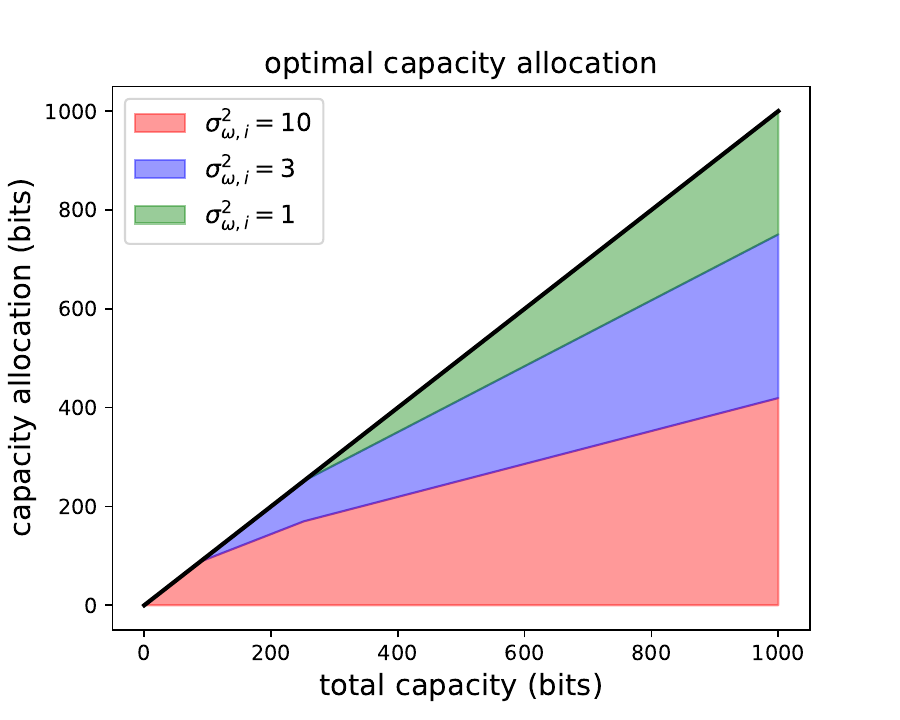}
    \caption{Optimal capacity allocation for Case 2}
    \label{fig:case_2_alloc}
\end{figure}

\begin{figure}[ht]
    \centering
    \includegraphics[scale=0.7]{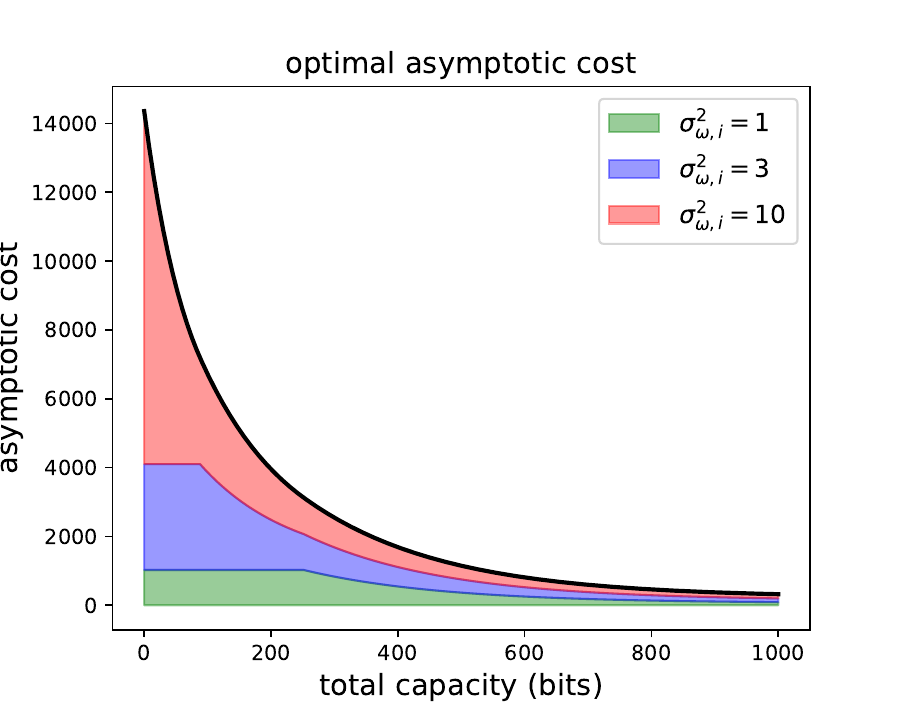}
    \caption{Optimal asymptotic loss for Case 2}
    \label{fig:case_2_cost}
\end{figure}

\subsubsection*{Case 3: Different observation signal-to-noise ratios}

We then consider a setting with different observation signal-to-noise (SNR) ratios. Specifically, we choose
$A =0.95I_{300}$, $C = \mathrm{diag}([I_{100}, I_{100}, 0.25 I_{100}])$,
$\Sigma_\omega = I_{300}$, and $\Sigma_\nu = \mathrm{diag}([0.1 I_{100}, 10 I_{100}, 10 I_{100}])$. Note that all scalar subsystems have the same system dynamics, but different observation SNRs.
Note that the optimal capacity allocation is illustrated in Figure~\ref{fig:case_3_alloc} while the optimal asymptotic cost is illustrated in Figure~\ref{fig:case_3_cost}.

\begin{figure}[ht]
    \centering
    \includegraphics[scale=0.7]{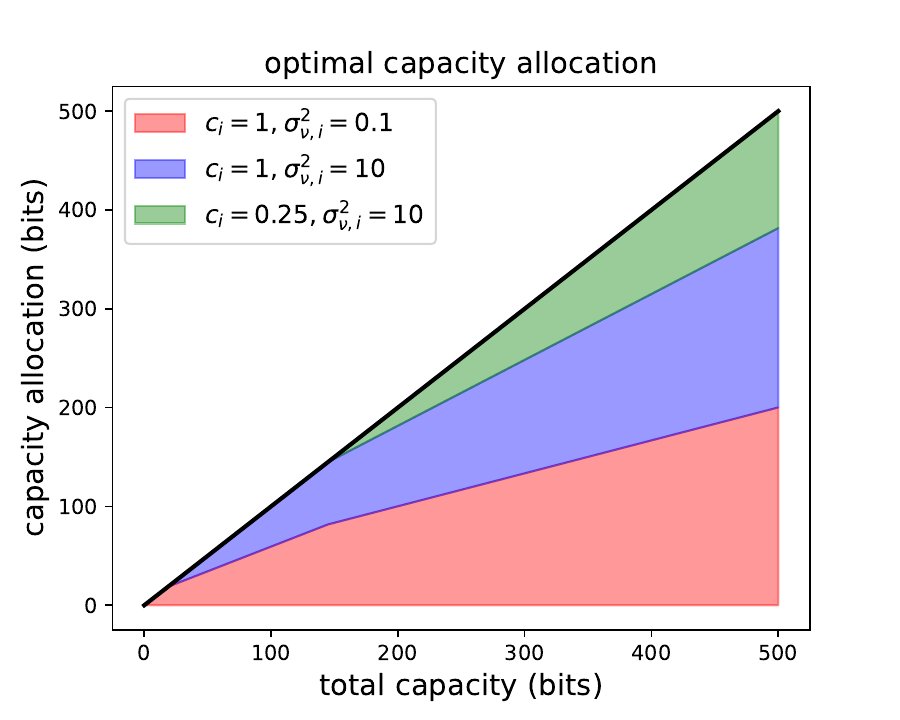}
    \caption{Optimal capacity allocation for Case 3}
    \label{fig:case_3_alloc}
\end{figure}

\begin{figure}[ht]
    \centering
    \includegraphics[scale=0.7]{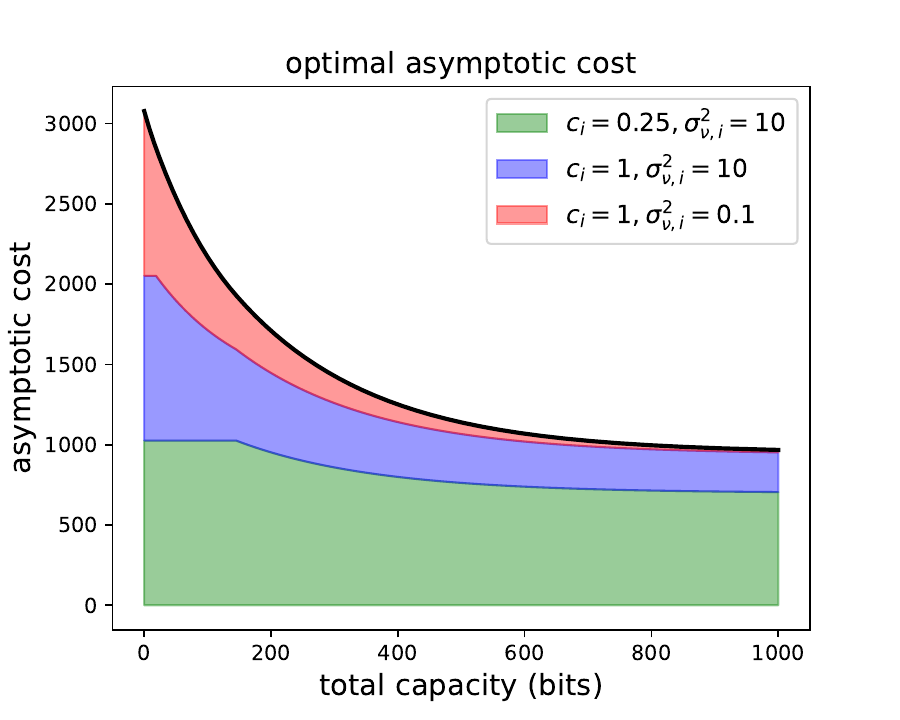}
    \caption{Optimal asymptotic loss for Case 3}
    \label{fig:case_3_cost}
\end{figure}

\subsubsection*{Case 4: Block-Diagonal Systems}

Finally, we consider a block-diagonal setting. Specifically, we choose $A = I_{100} \bigotimes \bar{A}$, where $\bigotimes$ denotes the Kronecker product and $\bar{A} \in \Re^{3 \times 3}$ is chosen as 
\begin{equation}
    \bar{A} = \left [
    \begin{array}{ccc}
       0.95  & 0.05 &  0 \\
       0  &  0.95 &  0 \\
       0 &  0 &  0.95
    \end{array}
    \right].
\end{equation}
We also choose $C = I_{300}$, $\Sigma_\omega=I_{300}$, and $\Sigma_\nu = I_{300}$. Note that $A$ is a block-diagonal matrix with $100$ $2 \times 2$ diagonal blocks and another $100$ $1 \times 1$ diagonal blocks.
Since $C$, $\Sigma_\omega$, and $\Sigma_\nu$ are also block-diagonal, this case is referred to as a block-diagonal case. Moreover, 
it is straightforward to show that $P$, $K$, and $L$ are also block-diagonal matrices.

\begin{figure}[ht]
    \centering
    \includegraphics[scale=0.7]{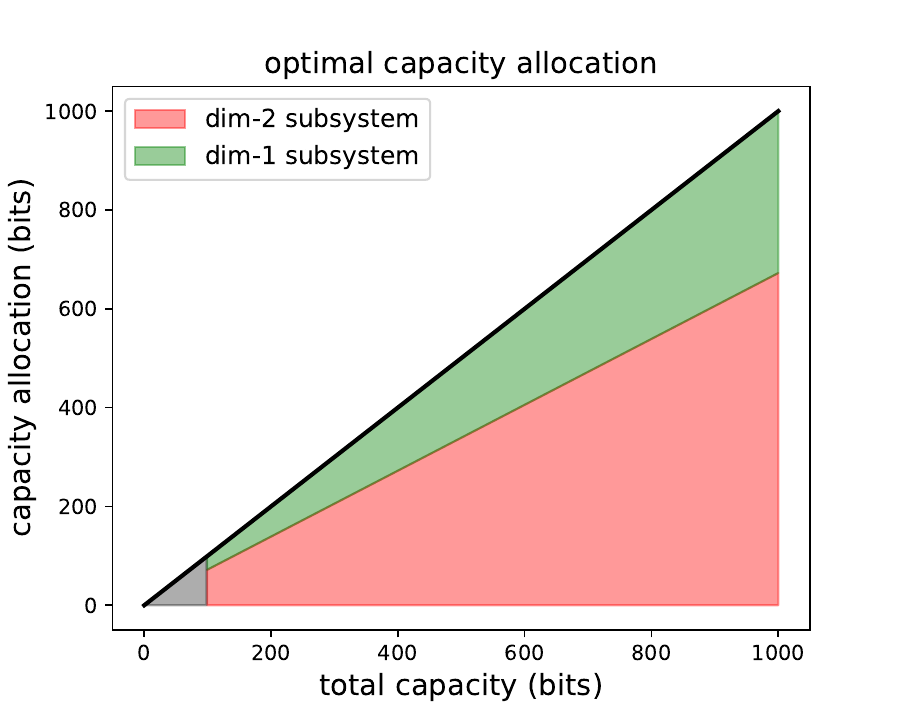}
    \caption{Optimal capacity allocation for Case 4}
    \label{fig:case_4_allocation}
\end{figure}

Figure~\ref{fig:case_4_allocation} plots the optimal capacity allocation across subsystems. Specifically, the red band corresponds to the capacity allocated to the dimension-2 subsystems, while the green band corresponds to the capacity allocated to the dimension-1 subsystems. Note that when the total capacity is small, the sufficient conditions of Theorem~\ref{thm:ss-ccc} do not hold. Thus, we do not know what the optimal capacity allocation is in such cases. In Figure~\ref{fig:case_4_allocation}, we use a grey band to denote such cases.

\begin{figure}[ht]
    \centering
    \includegraphics[scale=0.7]{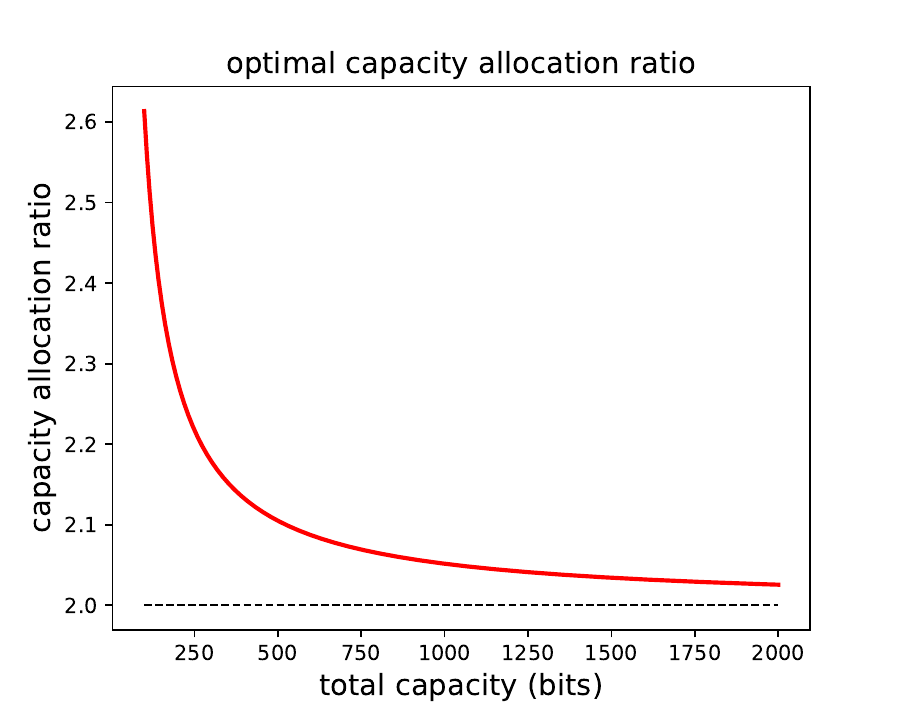}
    \caption{Optimal capacity allocation ratio for Case 4}
    \label{fig:case_4_allocation_ratio}
\end{figure}

Figure~\ref{fig:case_4_allocation_ratio} plot the ratio of the capacity allocated to dimension-2 subsystems to the capacity allocated to dimension-1 subsystems.
Note that in matrix $A$, the dimension-2 blocks and dimension-1 blocks have the same diagonal elements; however, the dimension-2 blocks also have an off-diagonal coupling. Intuitively, this off-diagonal coupling can help to reduce the cost under a given capacity, since it helps to share the information between two state variables. This has been justified in Figure~\ref{fig:case_4_allocation_ratio}. As we can see, the capacity ratio is always larger than $2$. Also, the smaller the total capacity is, the larger the capacity ratio is.


\section{Concluding Remarks}
\label{sec:conclusion}

The goal of this paper is to shed some light on the question how agents with limited capacity should allocate their resources for optimal performance. In particular, we study a simple yet relevant continual learning problem: the capacity-constrained linear-quadratic-Gaussian (LQG) sequential prediction problem.
We have derived a solution to this problem under appropriate technical conditions. Moreover, for problems that can be decomposed into a set of sub-problems, we have also demonstrated how to optimally allocate the capacity across these sub-problems in the steady state. 

We hope the results of this paper will motivate more future work along this research line. In particular, this paper has focused on a capacity-constrained sequential prediction setting. An interesting future direction is to derive solutions for capacity-constrained control/decision problems, such as capacity-constrained reinforcement learning.


\bibliography{reference}

\clearpage
\appendix

\section{Optimal agent without capacity constraints}
\label{app:without_capacity_constraint}

Let's define $\bar{\theta}_t = \E \left[ \theta_t | H_t \right]$ and $M_t = \Cov \left[ \theta_t | H_t \right]$. Notice that
\begin{align}
    \E \left[ \| \theta_t - \hat{\theta}_t \|_2^2 \, \middle | \, H_t \right ] =& \, \E \left[ \| \theta_t - \bar{\theta}_t + \bar{\theta}_t - \hat{\theta}_t \|_2^2 \, \middle | \, H_t \right ]  \nonumber \\
    =& \, \E \left[ \| \theta_t - \bar{\theta}_t  \|_2^2 \, \middle | \, H_t \right ] + \E \left[ \| \bar{\theta}_t - \hat{\theta}_t  \|_2^2 \, \middle | \, H_t \right ] + 2 
    \E \left[ (\theta_t - \bar{\theta}_t)^\top (\bar{\theta}_t - \hat{\theta}_t ) \, \middle | \, H_t \right ]
    \nonumber \\
    \stackrel{(a)}{=} & \, \E \left[ \| \theta_t - \bar{\theta}_t  \|_2^2 \, \middle | \, H_t \right ] + \E \left[ \| \bar{\theta}_t - \hat{\theta}_t  \|_2^2 \, \middle | \, H_t \right ],
\end{align}
where (a) follows from
\begin{align}
     \E \left[ (\theta_t - \bar{\theta}_t)^\top (\bar{\theta}_t - \hat{\theta}_t ) \, \middle | \, H_t \right ] \stackrel{(b)}{=}& \,  \E \left[ \theta_t - \bar{\theta}_t \, \middle | \, H_t \right]^\top
     \E \left[ \bar{\theta}_t - \hat{\theta}_t  \, \middle | \, H_t \right ] \nonumber \\
     =& \, (\bar{\theta}_t - \bar{\theta}_t) \E \left[ \bar{\theta}_t - \hat{\theta}_t  \, \middle | \, H_t \right ] = 0,
\end{align}
where (b) follows from the fact that $\bar{\theta}_t$ is deterministic given $H_t$, and $\hat{\theta}_t \perp \theta_t \, | \, H_t$. Consequently, the solution for optimal agent without capacity constraints is
\[
\hat{\theta}_t = \bar{\theta}_t = \E [\theta_t | H_t].
\]
Note that $\bar{\theta}_t$ and $M_t$ can be recursively computed by Kalman filtering. In particular, \begin{align}
    \bar{\theta}_{t+1} =& \, A \bar{\theta}_t + P_t C^\top (C P_t C^\top + \Sigma_\nu)^{-1} \left(Y_{t+1} - C A \bar{\theta}_t \right) \nonumber \\
    M_{t+1} =& \, P_t - P_t C^\top (C P_t C^\top + \Sigma_\nu)^{-1} C P_t,
\end{align}
where $P_t = A M_t A^\top + \Sigma_\omega$.

\section{Gaussian distortion-rate function}
\label{app:gaussian-dr}

For any positive-definite $\Sigma \in \Re^{d \times d}$ and any $B>0$, we consider the Gaussian distortion-rate function as:
\begin{align}
    D(B, \Sigma) =& \, \min_{\hat{\theta}} \E \left [ \| \hat{\theta} - \bar{\theta} \|_2^2 \right ] \\
    \text{s.t.} \quad & \, 
    \bar{\theta} \sim N(0, \Sigma) \nonumber \\
    & \,
    \I (\bar{\theta}; \hat{\theta}) \leq B  \nonumber
\end{align}
As we will see, the Gaussian distortion-rate function will play a key role in the analysis. In this appendix, we tried to compute the Gaussian distortion-rate function.

\subsection{Scalar case}
Let's first consider the scalar case where $d=1$ and $\Sigma=\sigma^2$. Then we have
\begin{align}
    B \stackrel{(a)}{\geq} & \, \I (\bar{\theta}; \hat{\theta}) = h(\bar{\theta}) - h (\bar{\theta} \, | \, \hat{\theta} ) = 
    h(\bar{\theta}) - h (\bar{\theta} - \hat{\theta} \, | \, \hat{\theta} ) \nonumber \\
    \stackrel{(b)}{\geq} & \,
    h(\bar{\theta}) - h (\bar{\theta} - \hat{\theta})
    \stackrel{(c)}{\geq} \frac{1}{2} \log \sigma^2 - \frac{1}{2} \log \E \left[(\bar{\theta} - \hat{\theta})^2 \right]
\end{align}
Thus, we have
\[
\E \left[ (\bar{\theta} - \hat{\theta})^2 \right] \geq \exp(-2B) \sigma^2,
\]
and this lower bound can be achieved. To see it, note that inequality (a) holds with equality if $B = \I (\bar{\theta}; \hat{\theta})$, inequality (b) holds with equality if $\bar{\theta} - \hat{\theta} \perp \hat{\theta}$, and inequality (c) holds with equality if $\bar{\theta} - \hat{\theta}$ is zero-mean Gaussian. Assume that
\[
\hat{\theta} = f \bar{\theta} + \epsilon, 
\]
where $\epsilon$ is independently drawn from $N(0, \phi^2)$. Clearly, inequality (c) holds with equality. For inequality (b) to hold with equality, we need 
$
\E \left[ (\bar{\theta} - \hat{\theta}) \hat{\theta} \right ] = 0
$, which is equivalent to
\[
\E \left[ \left( (1-f)\bar{\theta} - \epsilon \right) \left(f \bar{\theta} + \epsilon \right) \right ] = 0,
\]
that is
\[
(1 - f) f \sigma^2 - \phi^2 = 0.
\]
On the other hand, for inequality (a) to hold with equality, we need
\begin{align}
    B =& \, \I(\bar{\theta}, \hat{\theta}) = h(\hat{\theta}) - h(\hat{\theta} | \bar{\theta}) = h(\hat{\theta}) - h(\epsilon) = \frac{1}{2} \log \left( 1 + \frac{f^2 \sigma^2}{\phi^2} \right)
\end{align}
Solving the above equation, we have $f= 1 - \exp(-2B)$ and
$\phi^2 = \left[1 - \exp(-2B) \right] \exp(-2B) \sigma^2$.
Consequently, we have the following lemma:
\begin{lemma}
For the scalar Gaussian case, we have
$D(B, \sigma^2) = \exp(-2B) \sigma^2$.
\end{lemma}

\subsection{Diagonal case}
Let's consider the diagonal case when
$\Sigma = \mathrm{diag}(\sigma_1^2, \sigma_2^2, \ldots, \sigma_d^2 )$.
Then, it is equivalent to the following problem:
\begin{align}
    \label{opt:diagonal}
    \min_{B_1, \ldots, B_d} \quad & \, 
    \sum_{i=1}^d \exp(-2 B_i) \sigma^2_i \\
    \text{s.t.} \quad & \,  \sum_{i=1}^d B_i = B \; \text{and} \; B_i \geq 0 \quad \forall i = 1, 2, \ldots, d  \nonumber 
\end{align}

\noindent
To see that, notice that
\begin{align}
    B \geq & \,\I (\bar{\theta}; \hat{\theta}) = \, h(\bar{\theta}) - h (\bar{\theta} \, | \, \hat{\theta}) = \sum_{i=1}^d h(\bar{\theta}_i) - \sum_{i=1}^d h(\bar{\theta}_i \, | \, \bar{\theta}_{1:(i-1)}, \hat{\theta}) \nonumber \\
    \geq & \, \sum_{i=1}^d h(\bar{\theta}_i) - \sum_{i=1}^d h(\bar{\theta}_i \, | \, \hat{\theta}_i )
    = \sum_{i=1}^d \I (\bar{\theta}_i ; \hat{\theta}_i) 
\end{align}
Let $B_i = \I (\bar{\theta}_i ; \hat{\theta}_i) $, then we have
\[
\E \left[ \| \bar{\theta} - \hat{\theta} \|_2^2 \right] =
\sum_{i=1}^d \E \left[ (\bar{\theta}_i - \hat{\theta}_i )^2 \right] \geq \sum_{i=1}^d \exp(-2 B_i) \sigma_i^2
\]
Clearly this lower bound can be achieved. 
We now solve the optimization problem \ref{opt:diagonal}. Notice that for all $B_i >0$, the derivative of $\exp(-2 B_i) \sigma_i^2$ must be the same. Assume this derivative is $- \lambda$ where $\lambda >0$, then we have
\[
- 2 \exp(-2 B_i) \sigma_i^2 = - \eta \quad \text{for} \quad B_i >0.
\]
That is
\[
B_i = \frac{1}{2} \log \frac{2 \sigma_i^2}{\eta} \quad \text{if $B_i > 0$}
\]
On the other hand, $B_i=0$ if and only if the derivative at $B_i = 0$ is greater than or equal to $-\eta$, i.e. $-2 \sigma^2_i \geq - \eta$, that is $2 \sigma_i^2 / \eta \leq 1$. Hence
\[
B_i = \frac{1}{2} \left[ \log \frac{2 \sigma_i^2}{\eta}\right]^+ = 
\frac{1}{2} \log \left( \max \{ 2 \sigma_i^2 /\eta, 1 \} \right )
\]
where $\lambda$ is determined by the equation
\[
\sum_{i=1}^d B_i = \frac{1}{2} \sum_{i=1}^d  \left[ \log \frac{2 \sigma_i^2}{\eta}\right]^+ = B
\]
So for the diagonal case, we have
\[
D(B, \Sigma) = \sum_{i=1}^d \exp(-2 B_i) \sigma^2_i = \sum_{i=1}^d \frac{\sigma_i^2 }{\max \left \{ \frac{2 \sigma_i^2}{\lambda}, \, 1 \right \} } = \sum_{i=1}^d \min \left \{ \sigma^2_i, \frac{\eta}{2} \right \}
\]
Moreover, the optimal $\hat{\theta}$ is
$
\hat{\theta} = F \bar{\theta} + \epsilon$, where
\[
F = \mathrm{diag} \left( 1- \exp(-2B_1), 1- \exp(-2B_2), \cdots, 1- \exp(-2B_d) \right), 
\]
and $\epsilon$ is independently drawn from $N(0, \Phi)$, where
\[
\Phi = \mathrm{diag} \left(
\left[1 - \exp(-2B_1) \right] \exp(-2B_1) \sigma_1^2, 
\cdots,
\left[1 - \exp(-2B_d) \right] \exp(-2B_d) \sigma_d^2
\right).
\]

\subsection{General case}
Recall that $\bar{\theta} \sim N(0, \Sigma)$. Assume that $\Sigma$ is positive-definite, so we have $\Sigma = U \Lambda U^\top$, where $\Lambda$ is a diagonal matrix encoding the eigenvalues of $\Sigma$, and $U$ is an orthogonal matrix encoding the eigen-vectors of $\Sigma$. Note that
\[
\Cov \left[ U^\top \bar{\theta} \right] = \E \left[ U^\top \bar{\theta} \bar{\theta}^\top U \right] = U^\top \Sigma U 
= U^\top U \Lambda U^\top U = \Lambda.
\]
Moreover, we have
\[
\I \left(\bar{\theta}; \, \hat{\theta} \right) = \I \left(
U^\top \bar{\theta}; \, U^\top \hat{\theta} \right).
\]
and
\[
\left \|
\bar{\theta} - \hat{\theta}
\right \|_2^2 = \left \| U^\top \bar{\theta} - U^\top \hat{\theta} \right \|_2^2.
\]
Thus, the original problem reduces to 
\begin{align}
    D(B, \Sigma) =& \, \min_{U^\top \hat{\theta}} \E \left [ \| U^\top \hat{\theta} - U^\top \bar{\theta} \|_2^2 \right ] \\
    \text{s.t.} \quad & \, 
    U^\top \bar{\theta} \sim N(0, \Lambda) \nonumber \\
    & \,
    \I (U^\top\bar{\theta}; U^\top \hat{\theta}) \leq B  \nonumber
\end{align}

So we have the following theorem:
\begin{theorem}[Gaussian distortion-rate function]
\label{thm:gaussian-dr}
Let $\Lambda=\mathrm{diag}(\lambda_1, \ldots, \lambda_d)$ be the diagonal matrix encoding the eigenvalues of $\Sigma$, and $U$ denote the orthogonal matrix encoding the eigenvectors of $\Sigma$. Define $\eta>0$ as the unique solution of equation
$
\sum_{i=1}^d \left[ \log \left( 2 \lambda_i / \eta \right) \right]^+ = 2B$,
and $B_i = \frac{1}{2} \left[\log \left( 2 \lambda_i / \eta \right) \right]^+$ for all $i=1,2,\ldots, d$.
Then we have 
\[
D(B, \Sigma) = \sum_{i=1}^d \min \left \{ \lambda_i, \eta/ 2 \right \},
\]
which is achieved by
\[
\hat{\theta} = U F U^\top \bar{\theta} + U \epsilon,
\]
where 
\[
F = \mathrm{diag} \left( 1- \exp(-2B_1), 1- \exp(-2B_2), \cdots, 1- \exp(-2B_d) \right), 
\]
and $\epsilon$ is independently drawn from $N(0, \Phi)$, with
\[
\Phi = \mathrm{diag} \left(
\left[1 - \exp(-2B_1) \right] \exp(-2B_1) \lambda_1, 
\cdots,
\left[1 - \exp(-2B_d) \right] \exp(-2B_d) \lambda_d
\right).
\]

\end{theorem}

\end{document}

%% file: macro.tex
\newcommand{\E}{\mathbb{E}}
\newcommand{\I}{\mathbb{I}}
\newcommand{\Cov}{\mathrm{Cov}}
\newcommand{\Var}{\mathrm{Var}}
\newcommand{\trace}{\mathrm{tr}}

\newcommand{\diag}{\mathbf{diag}}